\documentclass[letterpaper, 10 pt, conference]{ieeeconf}
\usepackage[utf8]{inputenc}

\IEEEoverridecommandlockouts 
\let\labelindent\relax
\usepackage{enumitem}

\usepackage{amsmath}
\usepackage{amsfonts}

\usepackage{bbm}

\usepackage{tabularx,booktabs}
\usepackage{threeparttable}

\usepackage{graphicx}
\graphicspath{ {images/} }
\usepackage{subcaption}
\usepackage{lipsum} 
\usepackage{float}
\usepackage{cite}
\usepackage{multirow}
\usepackage{color}

\usepackage[ruled,vlined,linesnumbered]{algorithm2e}
\SetKwInOut{Parameter}{Parameters}
\SetKwComment{Comment}{$\triangleright$\ }{}

\usepackage{pgfgantt} % change figure captions

\newtheorem{theorem}{Theorem}

\newtheorem{proposition}[theorem]{Proposition}

\newtheorem{measure}{Measure}

\DeclareMathOperator*{\argmax}{arg\,max}

\usepackage[long]{optidef}

\makeatletter
\newcommand*\rel@kern[1]{\kern#1\dimexpr\macc@kerna}
\newcommand*\widebar[1]{%
  \begingroup
  \def\mathaccent##1##2{%
    \rel@kern{0.8}%
    \overline{\rel@kern{-0.8}\macc@nucleus\rel@kern{0.2}}%
    \rel@kern{-0.2}%
  }%
  \macc@depth\@ne
  \let\math@bgroup\@empty \let\math@egroup\macc@set@skewchar
  \mathsurround\z@ \frozen@everymath{\mathgroup\macc@group\relax}%
  \macc@set@skewchar\relax
  \let\mathaccentV\macc@nested@a
  \macc@nested@a\relax111{#1}%
  \endgroup
}
\makeatother

\begin{document}

% TITLE ----------------------------
\title{Radar-only ego-motion estimation in difficult settings via graph matching}
\author{Sarah H. Cen and Paul Newman
\vspace{-1em}\thanks{ S. H. Cen is with the Department of Electrical Engineering and Computer Science, Massachussetts Institute of Technology, Cambridge 02139, USA (shcen@mit.edu). P. Newman is with the Oxford Robotics Institute, University of Oxford, Oxford OX1 3PJ, UK (pnewman@robots.ox.ac.uk).}}

\maketitle

\begin{abstract} 
Radar detects stable, long-range objects under variable weather and lighting conditions, making it a reliable and versatile sensor well suited for ego-motion estimation. In this work, we propose a radar-only odometry pipeline that is highly robust to radar artifacts (e.g., speckle noise and false positives) and requires only one input parameter. We demonstrate its ability to adapt across diverse settings, from urban UK to off-road Iceland, achieving a scan matching accuracy of approximately 5.20 cm and 0.0929 deg when using GPS as ground truth (compared to visual odometry's 5.77 cm and 0.1032 deg). We present algorithms for keypoint extraction and data association, framing the latter as a graph matching optimization problem, and provide an in-depth system analysis. 
\end{abstract}
\vspace{-0.1em}

% INTRO ----------------------------
\section{Introduction}

Despite the rapid expansion and advancement of research into sensor systems for mobile autonomy, those used for intelligent transportation typically feature the same suite: vision, lidar, GPS, and proprioceptive sensors. Recent interest in radar for mobile autonomy \cite{adams2012robotic,cen2018mmwradar} has uncovered its potential as a highly robust, multipurpose sensor. Unlike vision and lidar systems, it is resilient to variable lighting and weather, and, unlike GPS, it functions both indoors and outdoors. Radar is reliable at short and long ranges, and it is becoming increasingly more compact and affordable. In addition to its practical benefits, radar returns are simultaneously data-efficient and information-rich: radar can observe multiple objects per transmission and detect their locations, velocities, and cross-section characteristics. 

As a result, radar is a highly versatile sensor that can withstand adverse conditions. Its ability to % consistently 
detect stable, long-range features makes it particularly suitable for odometry and localization. In this paper, we build on our previous work \cite{cen2018mmwradar}, in which we demonstrate precise ego-motion estimation using only a  millimeter-wave (MMW) frequency-modulated continuous-wave (FMCW) scanning radar.  As shown in \cite{cen2018mmwradar}, our system performs comparably to state-of-the-art visual odometry (VO) and GPS. Even under conditions that cause other sensors to fail, radar odometry (RO) is robust.

\setlength{\textfloatsep}{1em}
\begin{figure}[t]
    \centering
    \includegraphics[width=\columnwidth]{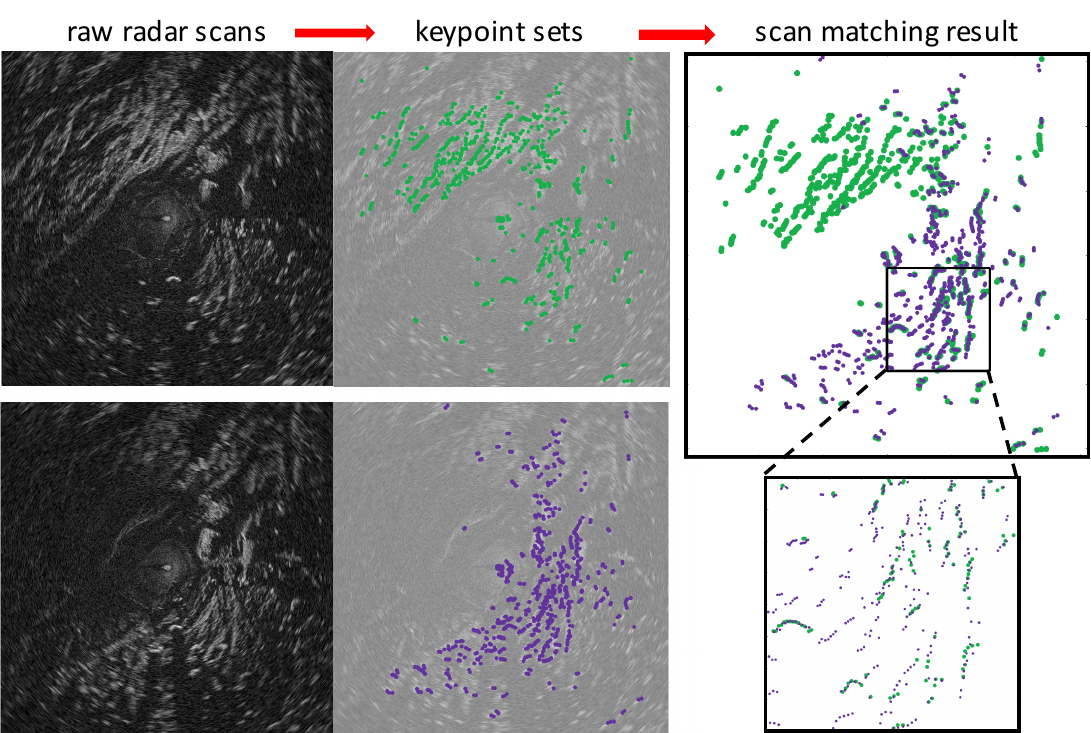}
    \caption{Demonstration of our one-parameter, radar-only odometry pipeline performing scan matching in a challenging scenario with no motion prior. On the left are two consecutive radar scans, which are 250$\times$250 m and taken 0.25 seconds apart while driving over uneven open terrain in Iceland. The keypoints extracted by our algorithm are shown in green and purple, and the alignment proposed by our data association algorithm is visualized on the right. Precise scan matching is achieved despite radar's high levels of measurement noise as well as significant appearance changes.\vspace{-0em}}
    \label{fig:first_page_pipeline}
\end{figure}

Common approaches to odometry for rangefinders have several key drawbacks, such as sensitivity to noise, reliance on other sensors, requiring \textit{a priori} knowledge, and utilizing restrictive models. In contrast, our radar-only system achieves precise odometry without model-reliant motion filters, outlier detection, or map creation. At its core is a geometrically-based, non-iterative scan matching algorithm that is displacement-independent and approximately optimizes a global objective function. Consequently, our pipeline works without a prior on the temporal or spatial relationship between keypoint sets and can handle significant measurement noise and missing detections, as shown in Fig. \ref{fig:first_page_pipeline}. 

\begin{figure*}[t!]
    \centering
    \includegraphics[width=\textwidth]{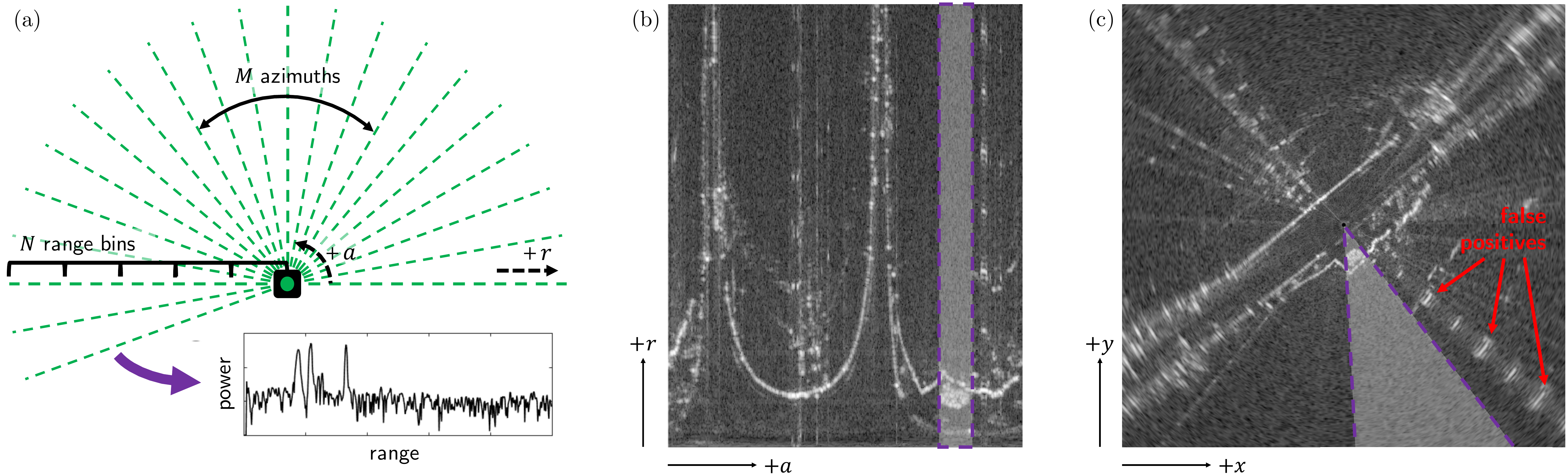}
    \caption{(a) Visualization of scanning radar operation from a bird's-eye view. The radar (green dot) sequentially observes $M$ azimuths (dotted lines), collecting readings for $N$ range bins along $r$. For each azimuth $a$, it outputs a power-range spectrum, an example of which is given on the bottom right. (b) The polar radar scan formed by combining the spectra of a full rotation of azimuths, as observed in Oxford city center, UK. (c) The corresponding Cartesian scan with equivalent regions in (b) and (c) highlighted. Note the high levels of speckle noise and false positives due to saturation.\vspace{-0.55em}}
    \label{fig:radar_op_and_scans}
\end{figure*}

In this paper, we present a revised system that features the same advantages enumerated above with several major improvements to \cite{cen2018mmwradar}. Specifically, our main contributions are: 
\begin{itemize}[noitemsep,topsep=0pt,partopsep=0pt]
\itemsep0em
    \item A new gradient-based, one-parameter keypoint extraction algorithm that functions well in diverse settings. 
    \item More efficient and robust data association.
    \item Successful real-time odometry in varied and challenging settings (e.g., off-road path in Iceland boulder field).
    \item In-depth theoretical analysis on system complexity and estimation uncertainty. 
\end{itemize}
Following a literature review in Section II, we present and analyze our pipeline in Sections III and IV. We provide experimental results and conclude in Sections V and VI. Both our works are presented and analyzed in detail in \cite{CenThesis}.

% RELATED WORKS ----------------------------
\section{Related works} \label{sec:related_works}

While visual \cite{aqel2016review,scaramuzza2011visual}, lidar \cite{zhang2017low,wolcott2017robust}, and wheel \cite{borenstein1996sensors} odometry are well studied, radar odometry remains challenging. Due to its wide spreadbeam and long range, radar has lower resolution than lidar and is highly susceptible to interference from clutter, which generates speckle noise. Radar scans also contain false positives %(see Fig. \ref{fig:radar_op_and_scans})
from multipath reflections and receiver saturation. As a result, radar odometry must be robust to measurement noise and false detections, and it must demonstrate high precision despite low-resolution data and slow update speeds. Odometry methods for radar can be categorized as indirect or direct. 
Indirect methods first extract salient keypoints, then associate those that correspond to the same location. Direct methods \cite{novak1978correlation,checchin2010radar,kellner2013instantaneous}, which forego keypoint extraction and operate on minimally pre-processed sensor outputs, are discussed in \cite{CenThesis,cen2018mmwradar} and not in this paper. All methods assume the majority of observed objects are static.

The first step of indirect methods is keypoint extraction, for which the most popular approach is constant false-alarm rate (CFAR) detection \cite{skolnik2008radar}, which distinguishes peaks from noise using sliding-window thresholding. CFAR and its variants generally require at least three tunable parameters, which are based on assumed noise characteristics and do not behave consistently across datasets (see \cite{cen2018mmwradar} for comparison). Some works leverage the knowledge that coherent structures make good keypoints by clustering or detecting the edges of  bright regions in scans \cite{werber2016interesting}. Others elect to represent the surroundings using predetermined geometric primitives \cite{lu1997robot} % cite{yang2002laser}
or models, like the normal distribution transform (NDT) \cite{rapp2015clustering}.  Vision-inspired works treat radar scans as images and extract features, like SIFT and FAST \cite{callmer2011radar}. 

These keypoints must then undergo data association, also known as scan matching \cite{adams2012robotic,lu1997robot} in robotics. The most common technique is iterative closest point (ICP) \cite{besl1992method}, which iteratively matches points using naive methods until the alignment between keypoint sets is sufficiently close \cite{lu1997robot,yang2002laser}. ICP relies on a good estimate of the relative displacement (i.e., motion prior) between scans. Other data association techniques search for  motion parameters that optimize some objective function, like maximizing similarity (e.g., overlap of Gaussian distributions for NDT \cite{rapp2017probabilistic}) or minimizing distance (e.g., cluster edge difference \cite{novak1978correlation}) between keypoint sets. Two further examples of objective functions characterize map quality \cite{chandran2006motion} and  radar scan distortion \cite{vivet2013localization} in terms of motion. Feature-based approaches associate keypoints using descriptors, like BASD \cite{schuster2016landmark}  and SURF \cite{callmer2011radar}. 

Many of the methods discussed do not generalize well due to the high levels of noise in radar scans. Adequate performance often requires fine tuning, \textit{a priori} knowledge, restrictive assumptions, or outlier detection. Several of the works rely heavily on other sensors for robustness, which compromises performance under conditions that cause these sensors to fail, or use simultaneous localization and mapping (SLAM),  which is accompanied by overhead costs and model-reliant motion filters \cite{schuster2016robust,callmer2011radar}. 

We utilize an indirect method to explicitly select salient information from noisy artifacts. We present two algorithms that improve on our previous work \cite{cen2018mmwradar}. The new pipeline is more efficient and easily adaptable across diverse settings. Accordingly, our keypoint extraction algorithm returns interpretable features with minimal redundancy, only one input parameter, and no assumptions about the scene structure or noise. Our data association algorithm does not need a motion prior or parameter tuning, and it is robust to large amounts of noise and false detections. An in-depth literature review and description of our contributions can be found in \cite{CenThesis}. 

\section{Radar-only ego-motion estimation} \label{sec:radar_only_egomotion_est}

\subsection{FMCW scanning radar}
% \subsection{MMW FMCW scanning radar}
\label{sec:radar_description}

FMCW radar, which is becoming more compact, safe, and affordable than alternative radars, collects long-range measurements with high accuracy and remains resilient under variable lighting and weather \cite{CenThesis}. Scanning radar, as shown in Fig. \ref{fig:radar_op_and_scans}, sequentially observes narrow angular regions as it rotates, allowing it to locate an object that falls inside the transmitted beam by its range and azimuth. Received power depends on object reflectivity, surface area, orientation, and material. A wide spreadbeam in elevation and long wavelength allow multiple objects to be detected per transmission. 

For each azimuth, radar outputs a one-dimensional signal, termed the power-range spectrum, which encodes the power reflected by the scatterers within the beam at each range. After a full rotation, radar returns a two-dimensional scan, as shown in Fig. \ref{fig:radar_op_and_scans}. Radar scans are both information-rich and data-efficient, but they contain several unwanted artifacts visible in Fig. \ref{fig:radar_op_and_scans}, including noise and false detections \cite{CenThesis}. %(from multipath reflections and receiver saturation). 
Another consideration of radar is its lower resolution and slower measurement update speeds compared to lidar. 

\begin{algorithm}[t]
 \KwIn{Radar scan $S \in \mathbb{R}^{m \times n}$} 
 \KwOut{Set of keypoints $L(S) \in \mathbb{R}^{p \times 2}$}
 \Parameter{Maximum number of keypoints $\ell_{\text{max}}$}
 \BlankLine
 $G \leftarrow \text{computeNormalizedGradientMagnitude(S)}$\label{line:gradient_mag}\\
 $S' \leftarrow S - \text{mean}(S)$\\ \label{line:subtract_mean}
 $H \leftarrow (1 - G) \times S' $\label{line:scal_img}\\
 $I \leftarrow \text{getIndicesOfElementsInDescendingOrder}(H)$\label{line:sort_indices}\\
 Initialize marked matrix $R \leftarrow [ \textit{false} ]^{m \times n}$, counter $\leftarrow \ell$ \label{line:initialize_marked_region}\\
 \While{\upshape{($\ell < \ell_{\text{max}}$) and (any \textit{false} in $R$)}}{
    $(a,r) \leftarrow \text{getNextIndices}(I)$ \label{line:get_index}\\
    \If{\upshape{not $R(a,r)$}} {
        $(r_{\text{low}}, r_{\text{high}}) \leftarrow \text{findRangeBoundaries}(S'[a,:])$ \label{line:range_bounds}\\ 
        \If{\upshape{none in $R[a,r_{\text{low}}:r_{\text{high}}]$}} {
            Increment $\ell$ \label{line:increment_counter}
        }
        $R[a,r_{\text{low}}:r_{\text{high}}] \leftarrow \textit{true}$ \label{line:mark_region}
    }
 }
 \For{\upshape{$a$ from $1$ to $m$}\label{line:for_add}}{
    \For{\upshape{each \upshape{marked region $Q$ in $R[a,:]$}}}{
        \If{\upshape{any in $R[a-1,Q] \cup R[a+1,Q]$}}{
            $L \leftarrow {L \cup (a,r)}$ for $H[a,r] = \text{max}(H[a,Q])$ \label{line:add_keypoint}
        }
    }
 }
  \caption{Keypoint Extraction} \label{alg:KE}
\end{algorithm}

\subsection{Keypoint extraction}
\label{sec:KE}

\begin{figure}[t]
    \centering
    \includegraphics[width=\columnwidth]{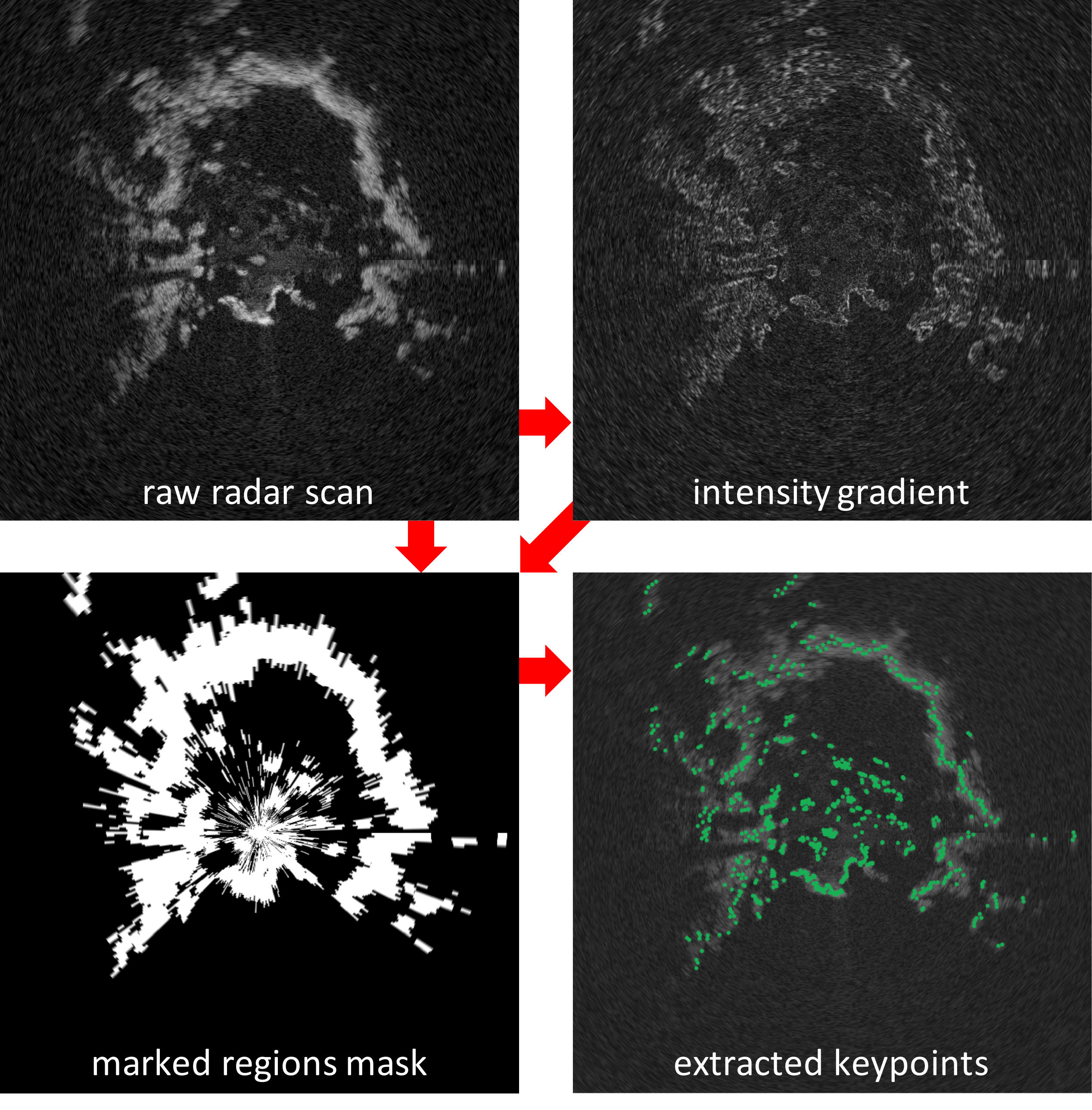}
    \caption{Visualization of our keypoint extraction approach operating in a forest in Iceland. From the 250$\times$250 m radar scan (top left) and its intensity gradient (top right), important regions in the scan are identified. Those marked (bottom left), the number of which is capped by the only tunable parameter $\ell_{\text{max}}$, are used to produce a keypoint set (bottom right) that captures the scene structure while minimizing redundant keypoints.
    }
\end{figure}

In this section, we present a keypoint extraction algorithm that is designed to: 
\begin{itemize}[noitemsep,topsep=0pt,partopsep=0pt]
    \item Return interpretable and meaningful keypoints.
    \item Be straightforward to use and require only one input parameter specifying the maximum number of keypoints.
    \item Adapt to diverse settings without \textit{a priori} knowledge, like the expected noise characteristics.
    \item Avoid redundant returns to improve data efficiency.
    \item Ignore bright swathes due to rolling, pitching, or clutter.
\end{itemize}
The procedure is given in Algorithm \ref{alg:KE}. It takes as an input the raw radar scan, which contains the power readings for $m$ azimuths and $n$ range bins, and the desired maximum number of keypoints $\ell_{\text{max}}$. The radar scan is re-scaled to favor areas of high intensity and low gradients (using the Prewitt operator), storing the result in $H$ (lines \ref{line:gradient_mag}-\ref{line:scal_img}). The algorithm then iterates through points $(a,r)$ in order of highest to lowest $H(a,r)$ and marking off adjacent regions in $R$ that should be ignored in future iterations. These regions are identified by searching $S'$ along $a$ for the closest range indices $(r_{\text{low}}, r_{\text{high}})$ below and above $r$ with values less than $0$ (lines \ref{line:sort_indices}-\ref{line:mark_region}). After $\ell_\text{max}$ regions are marked, iterate through $a$'s. For each  continuous marked region in $a$, the range bin $r$ with the highest value is added to keypoint set $L$. Isolated detections, which are more likely to be speckle noise, are eliminated when a marked region has no adjacent neighbors in azimuth (lines \ref{line:for_add}-\ref{line:add_keypoint}). 

The algorithm returns meaningful keypoints by seeking peaks in the scan (areas of high intensity and low gradient). It minimizes redundancy by returning only one keypoint per continuously marked sequence of range bins. This strategy also prevents large bright regions due to rolling, pitching, or clutter (RPC) from producing many keypoints that can distract from important keypoints corresponding to stable objects. Furthermore, by minimizing the amount of prior knowledge required and seeking coherent structures irrespective of noise, the algorithm adapts to diverse settings with varying levels of noise. The only input parameter is an upper bound on the number of desired keypoints. 

In \cite{CenThesis}, we discuss how CFAR is not ideal for our objectives. It requires multiple tunable parameters, provides redundant returns, and often struggles with bright patches caused by RPC.  In comparison to our previous work \cite{cen2018mmwradar}, our new algorithm has fewer parameters and operates well even in unstructured environments, like forests.

%%%%%%%%%%%%%%%%%%%%%%%%%%%%%%%%%%%%%%%%%%%%%%%%%
\subsection{Data association}
\label{sec:DA}

In this section, we present an improved data association algorithm, which is revised from \cite{cen2018mmwradar} and designed to: 
\begin{itemize}[noitemsep,topsep=0pt,partopsep=0pt]
    \item Have zero tunable parameters. 
    \item Work without prior knowledge on the relative displacement (translational or rotational) between scans.
    \item Handle high levels of measurement noise, false (or ghost) objects, and other appearance changes.
    \item Utilize global information across the scans to generate correspondences that are mutually consistent. 
\end{itemize}
Our data association algorithm is an example of graph matching \cite{leordeanu2005spectral}. It seeks point correspondences based on global geometric constraints. Unlike ICP, which relies on the approximate alignment of scans, our method does not need prior knowledge of the relative displacement between scans. With no parameter tuning, it is easy to apply across datasets. Moreover, by formulating data association as a global optimization problem, our approach utilizes information across the entire scan and naturally removes outliers. Compared to \cite{cen2018mmwradar}, our current method is more efficient and robust due to changes to the unary keypoint descriptor and termination condition. As given in Algorithm \ref{alg:DA}, our approach takes in two keypoint sets $L_1$ and $L_2$ and proceeds in three stages.  

\begin{algorithm}[t]
 \KwIn{Keypoint sets $L_1$ and $L_2$, where $|L_1| \leq |L_2|$}
 \KwOut{Selected matches $M(L_1,L_2)$}
 \Parameter{Azimuth and range resolutions $\alpha$ and $\rho$, respectively, of the radar being used}
 \BlankLine
 $U^{u \times 2} \leftarrow \text{unaryMatchesFromDescriptors}(L_1,L_2,\alpha,\rho)$ \label{line:unarymatches}\\
 $C^{u \times u} \leftarrow \text{pairwiseCompatibilityScores}(U,L_1,L_2)$ \label{line:pairwisecompat1}\\ 
 $\mathbf{v}^* \leftarrow \text{principalEigenvector}(C)$ \label{line:maxeigenvec}\\
 Initialize empty set $\mathcal{M}$ and $\hat{\mathbf{m}} \leftarrow [0]^{u \times 1}$\label{line:matchinit}\\
 Initialize $\textit{unsearched} \leftarrow \{1, 2, \hdots, u\}$ and $\textit{score} \leftarrow 0$\label{line:unsearchedinit}\\
 \While{\upshape{\textit{unsearched} is not empty}\label{line:whilecond}}{
    $\hat{{m}}_g \leftarrow 1$ given $g$ s.t. ${v}^{*2}_g \geq {v}^{*2}_h \hspace{2pt} \forall \hspace{2pt} h \in \textit{unsearched}$ \label{line:maxrewardmatch}\\
  
    Terminate if $\hat{\mathbf{m}}^\top C \hat{\mathbf{m}} < \textit{score}$ \label{line:terminationcond}\\
    Add the match $U[g,:]$ to $\mathcal{M}$ \label{line:addmatch}\\
     Remove values $h$ from $\textit{unsearched}$ if  $(U[g,1] = U[h,1] \cup U[g,2] = U[h,2])$ \label{line:alternativematches1}
 }
  \caption{Data Association} \label{alg:DA}
\end{algorithm}

\textit{Stage 1}: Without loss of generality, assume $u = |L_1| \leq |L_2|$. A unary candidate proposal method (see \cite{cen2018mmwradar}) matches each keypoint $i_1 \in L_1$  to another $i_2 \in L_2$, and the pair is added to $U$ (line \ref{line:unarymatches}). In this paper, we present a new rotation-invariant descriptor $d^i \in [0,1]^{\alpha + \rho}$ for keypoint $i$. To compute $d[1$:$\alpha]$, we create a histogram of the number of elements in each angular slice around keypoint $i$, take its fast Fourier transform, then normalize its phase. To compute $d[\alpha+1$:$\alpha+\rho]$, we create a histogram of the number of elements in each annulus around keypoint $i$, then normalize it. The angular slices and annuli have the same spatial resolution as the radar scan. Note that, as done in \cite{cen2018mmwradar}, each element contributing to the descriptor is weighted by its range relative to the radar to correct for radar's range-density bias. 

\textit{Stage 2}: To refine the match proposals in $U$, each pairwise combination of keypoints $g = (i_1,j_1) \in L_1$ is compared to its associated (according to $U$) match pair $h = (i_2,j_2) \in L_2$. If $i_1$ and $i_2$ do indeed represent identical keypoints in the environment and the same holds true for $j_1$ and $j_2$, then the pairwise relationships $g$ and $h$ must be similar, which is quantified by the compatibility score $C_{gh} \in [0,1]$ of the non-negative symmetric compatibility matrix $C$ (line \ref{line:pairwisecompat1}). The formulation of this score is given in \cite{cen2018mmwradar}. 

\textit{Stage 3}: Maximizing the global compatibility score is done by finding the optimal match vector $\mathbf{m}^*$ such that: 
\begin{align}
    \mathbf{m}^* = \argmax_{\mathbf{m} \in \{0,1\}^{u \times 1}} \bigg(\frac{\mathbf{m}^\top C \mathbf{m}}{\mathbf{m}^\top  \mathbf{m}} \bigg). \label{eq:optim}
\end{align}
Logically, $\mathbf{m}$ must satisfy the (i) integrality constraint: all elements are either $0$ or $1$ (i.e., a match is true or false) and (ii) uniqueness constraint (i.e., each keypoint in $L_1$ cannot maps to more than one  keypoint in $L_2$, and vice versa). We define $\mathbf{m}$ such that  $m_k = 1$ if $U_k = (i_1,i_2) \in M$ and $m_k = 0$ otherwise, and let $\mathcal{M}: L_1 \to L_2$ be the mapping between keypoint sets according to $M$. This formulation is known as graph matching \cite{leordeanu2005spectral}, and it is NP-hard, so we relax the constraints %\eqref{mathc:integral}-\eqref{mathc:1to1_backward} 
and solve instead for $\mathbf{v}^*$ such that:
\begin{align}
    \mathbf{v}^* = \argmax_{\mathbf{v} \in \mathbb{R}^{u \times 1}} \bigg(\frac{\mathbf{v}^\top C \mathbf{v}}{\mathbf{v}^\top  \mathbf{v}} \bigg). \label{eq:relax_max}
\end{align}
By the Raleigh's ratio and Perron-Frobenius theorems \cite{leordeanu2005spectral}, the solution $\mathbf{v}^* \in [0, 1]^{u \times 1}$ is the principal eigenvector of $C$. We use $\mathbf{v}^*$ to approximate the optimal constrained solution to \eqref{eq:optim} via the greedy approach in lines \ref{line:matchinit}-\ref{line:alternativematches1}. We present a new termination condition that stops adding matches when the new one is incompatible with those already in $M$, as explained in Section \ref{sec:theoretical_analysis}. Compared to our previous work \cite{cen2018mmwradar}, the new condition is more robust and consistent because it does not depend on the magnitude of elements in $\mathbf{v}^*$. 

\begin{figure}[t]
    \centering
    \includegraphics[width=\columnwidth]{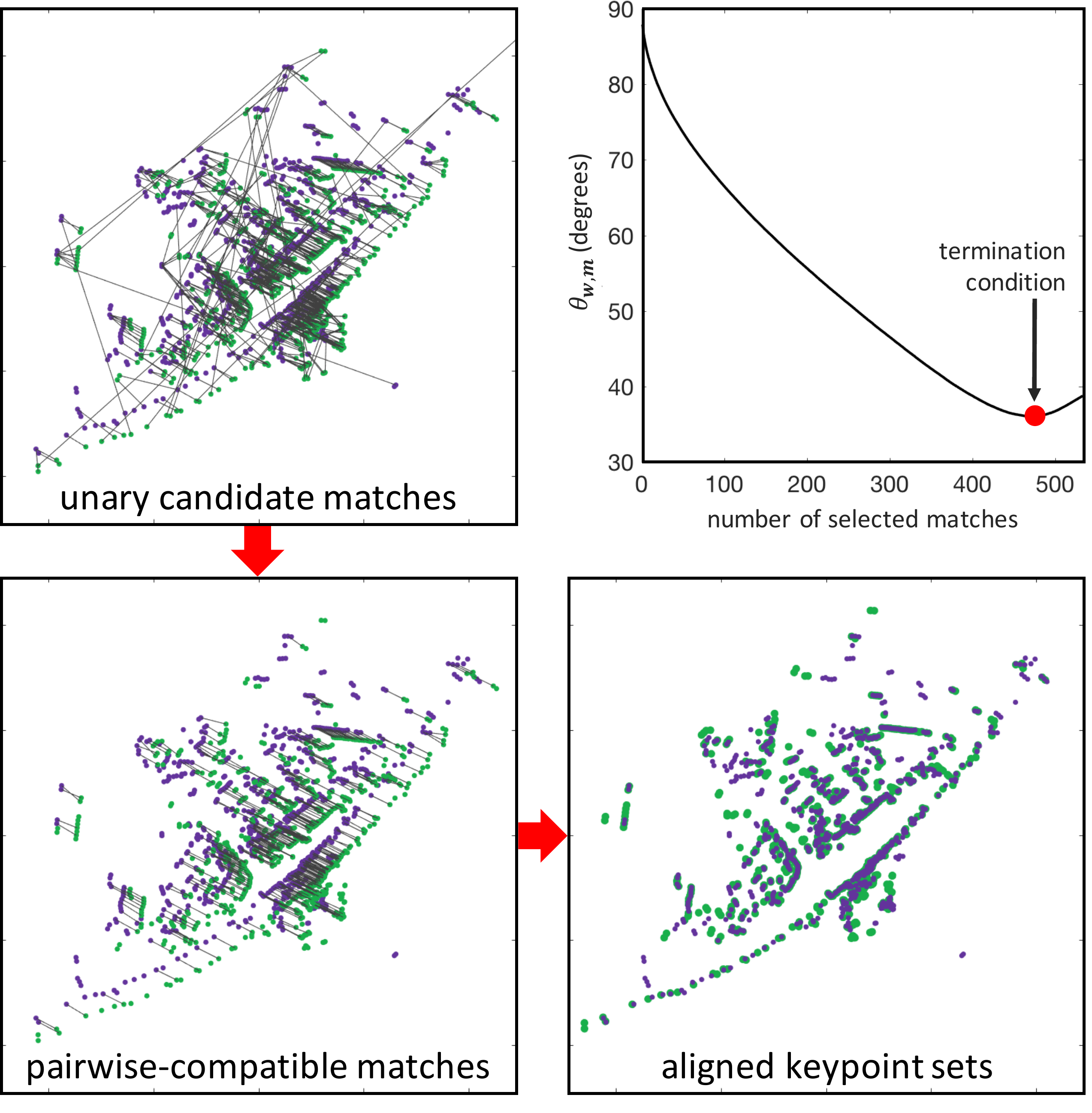}
    \caption{Visualization of our data association approach operating in Warwick, UK. It first proposes unary matches using feature descriptors (top left), then uses pairwise relationships to identify  mutually consistent matches (bottom left). Our method approximates the solution to an NP-hard problem using a greedy algorithm, for which the termination condition is visualized on the top right. The final alignment (bottom right) shows the robustness to noise.
    }
\end{figure}
The rigid-body motion (RBM) is found using a well-known  technique \cite{challis1995procedure}. As radar scans are in 2D, ego-motion is restricted to 3-DOF: planar translation and rotation.  

%%%%%%%%%%%%%%%%%%%%%%%%%%%%%%%%%%%%%%%%%%%%%%%%%%%
\section{System analysis}

%%%%%%%%%%%%%%%%%%%%%%%%%%%%%%%%%%%%%%%%%%%%%%%%%%%
\subsection{Algorithmic complexity}
The complexity of our system is summarized in Table \ref{tab:complexity}. For most radars, $A < R$. To estimate the average complexity, it is assumed that $A < \ell_{\text{max}} < R$. The second row gives the complexity for data association with no prior on the relative motion between scans. When running the system in real time, because the relative motion between consecutive time steps is limited, a mild prior can be  used to improve computational efficiency by reducing the number of unary descriptor comparisons. For this analysis, the intensity gradient, ordering, and principal eigenvector are computed using the Prewitt operator, merge sort, and power iteration, respectively. 

Overall, the system computation is generally dominated by the unary candidate matching step in \textit{Stage 1} of data association, such that the average system complexity is given by $\Theta (\ell_\text{max}^2(A + R)) \sim \Theta (\ell_\text{max}^3)$. Because the bottleneck depends on $\ell_\text{max}$, we intentionally allow $\ell_\text{max}$ to be set by the user as the \textit{only} tunable input parameter of our system.

%%%%%%%%%%%%%%%%%%%%%%%%%%%%%%%%%%%%%%%%%%%%%%%%%%%
\subsection{Theoretical analysis and performance measures}\label{sec:theoretical_analysis}

The fundamental principle behind our system is the optimization problem \eqref{eq:optim}. % -\eqref{mathc:1to1_backward}. 
Intuitively, it seeks the set of matches that maximizes the global compatibility score $\mathcal{G}({\mathbf{x},C}) = \frac{\mathbf{x}^\top C \mathbf{x}}{\mathbf{x}^\top  \mathbf{x}}$ in \eqref{eq:optim}. Because this problem is NP-hard, our algorithm approximates the optimal solution. This section investigates how well solution $M$ approximates $M^*$. Let $C = C^* + E$, where $C_{ij}^* = C_{ij}$ if and only if $U_i,U_j \in M^*$, and $C_{ij}^* = 0$ otherwise. $C^*$ is the corrected $C$ in which $k \notin M^*$ does not inflate the compatibility scores of $i \in M^*$, and $E$ can be considered a perturbation from $C^*$ due to incorrect matches.

\begin{proposition}
$\mathcal{G}({\mathbf{m},C})$ always overestimates the true global compatibility score, except when all proposed matches are correct (i.e., $U \in M^*$), in which case they are equal.
    \label{lemma_optimistic}
\end{proposition}
\begin{proof} 
By definition of $C$ and $C^*$, $E$ is non-negative: $\mathbf{x}^\top E \mathbf{x} \geq 0$ for any $\mathbf{x} \geq 0$ $\Rightarrow$ $\mathbf{m}^\top C \mathbf{m} =  \mathbf{m}^\top C^* \mathbf{m} + \mathbf{m}^\top E \mathbf{m} \geq \mathbf{m}^\top C^* \mathbf{m}$, or $\mathcal{G}({\mathbf{m},C}) \geq \mathcal{G}({\mathbf{m},C^*})$. 
\end{proof}

Proposition \ref{lemma_optimistic} implies that match selection is overly optimistic, and it is more likely to add too many matches, some of which are bad, than to not add good ones. Therefore, when using $\mathbf{v}_1$, we seek to: (1) resolve conflicting proposed matches in order to satisfy the constraints on \eqref{eq:optim} by selecting the best of them, and (2) avoid adding poor matches. Our algorithm achieves the first goal by adding matches in order of largest corresponding value in $\mathbf{v}_1$. It achieves the second with the termination condition, explained next. Let $\overline{C}$ be the compatibility matrix assuming perfect measurements. 
\begin{proposition}
   Assuming no perfect symmetry exists in the keypoint sets, $\mathbf{m}^* = \argmax_{\mathbf{x} \in [0,1]^{u}} \mathcal{G}({\mathbf{x},\overline{C}^*})$. 
    \label{lemma_max_vec}
\end{proposition}
\begin{proof} 
By the definition of compatibility scores \cite{cen2018mmwradar}, the maximum score $\overline{C}_{ij}^* = 1$ if $i,j \in M^*$, and $0 \leq \overline{C}^*_{ij} < 1$ otherwise. Since we assume perfect measurements and no symmetries, to meet the uniqueness constraint in \eqref{eq:optim}, $m_i^* \neq 1 \Rightarrow \overline{C}^*_{ij} \neq 1 \forall j \in \{1, \hdots, u\}$. By the integral constraint, $m_i^* \neq 1 \Rightarrow m_i^* = 0$, and  $m_i^* = 1$ if $i \in M^*$. 
\end{proof} 

By Proposition \ref{lemma_max_vec}, given the true compatibility matrix $\overline{C}^*$ and optimal match set $M^*$, the solution is, in fact, the binary vector $\mathbf{m}^*$ even though no integral constraints are placed on $\mathbf{x}$. Since all accepted matches in $\mathbf{m}^*$ are equally weighted, they are equally valued and agree with one another. We use this fact to seek a solution $M$ that, like $M^*$, contains matches that are mutually consistent. By Proposition \ref{lemma_optimistic}, the global compatibility score is an optimistic measure (i.e., $\mathcal{G}({\mathbf{m^1},C}) > \mathcal{G}({\mathbf{m^2},C})$ does not imply that matches in $M^1$ are better than those in $M^2$). We use Proposition \ref{lemma_max_vec} to construct a more reliable measure, as follows. 

\begin{measure}
The index $\mathcal{I}_{M} = \text{cos}(\theta_{\mathbf{w},\mathbf{m}}) \in [0,1]$ measures the mutual compatibility of selected matches $M$, where $\theta_{\mathbf{w},\mathbf{m}}$ is the angle between vectors $\mathbf{w} = C (\mathbf{m} \odot \mathbf{v}_1)$ and $\mathbf{m}$.
\end{measure}

Using this property, we terminate the greedy algorithm when the next match would cause $\mathcal{I}_\mathcal{M}$ to decreases. Intuitively, the new match would cause $\mathbf{w}$ to move farther from its corresponding binary solution, where this separation is inversely related to the consistency between matches. Although similar to $\text{cos}(\theta_{\mathbf{m},C\mathbf{m}}) = \mathcal{G}_{\mathbf{m},C}$, the continuous values ${v}_{1,i}$ in $\mathcal{I}_\mathcal{M}$ scale match $i$'s contribution by its compatibility.  

\begin{measure}
Assume the chosen solution is the optimal one: $M = M^*$. Let $\lambda_i^*$ represent the $i$-th largest eigenvalue of $C^*$. Then, the normalized eigengap $\mathcal{I}_{\text{eg}} = (\lambda^*_1 - \lambda^*_2)/u \in [0,1]$ measures the solution's robustness to perturbations. 
\end{measure}

While $\mathcal{I}_{M}$ is introspective (quantifies how well the selected matches agree), the second measure $\mathcal{I}_{\text{eg}}$ considers all proposed matches $U$ and quantifies the system's confidence it its solution $M^*$ over the next-highest scoring solution.

\begin{table}
\begin{center}
\caption{Algorithmic time complexity}\label{tab:complexity}
\begin{threeparttable}
\begin{tabularx}{\columnwidth}{@{\extracolsep{\fill}}l l l l @{}}
  \toprule
   & Average & Worst
  \\
  \midrule
  KE\tnote{1} 
  & $\Theta(AR \text{ log}(AR) + \ell_{\text{max}})$  & $O(AR \text{ log}(AR) + \ell_{\text{max} }R)$ \\ %& $\Omega(AR \text{log}(AR) + \ell_{\text{max}})$  \\
  DA%\tnote{2} 
  & $\Theta(\ell_\text{max}^2 (A + R))$  & $O(\ell_\text{max}^2 (A + R + \text{log}(\ell_\text{max}) ))$ \\%& $\Omega()$ \\
   DA\tnote{2} %\tnote{3} 
   & $\Theta(\ell_\text{max}^2 \text{ log}(\ell_\text{max}))$  & $O(\ell_\text{max}^2 \text{ log}(\ell_\text{max}) )$ \\%& $\Omega()$ \\
  ME\tnote{3}%
  & $\Theta(\ell_\text{max}^2)$  & $O(\ell_\text{max}^3)$ \\%& $\Omega()$ \\
  \bottomrule
\end{tabularx}
\begin{tablenotes}[flushleft]\footnotesize
    \item[1] Keypoint extraction
    \item[1] Data association with mild prior on relative motion
    \item[2] Rigid-body motion estimation
\end{tablenotes}
\end{threeparttable}
\vspace{-0.5em}
\end{center}
\end{table}

%%%%%%%%%%%%%%%%%%%%%%%%%%%%%%%%%%%%%%%%%%%%%%%%%%%
\section{Results} \label{sec:results}

\begin{figure*}[t!]
    \centering
    \begin{subfigure}[t]{0.45\textwidth}
        \centering
        \includegraphics[width=\textwidth]{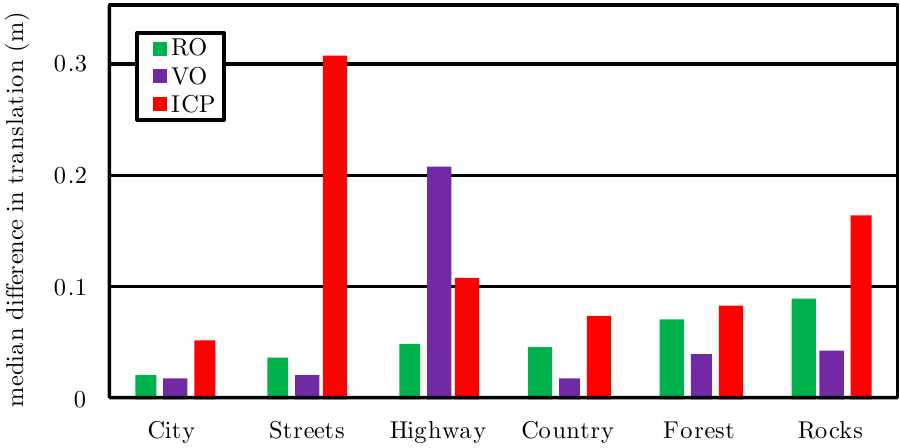}
        \caption{}
    \end{subfigure}%
    \hspace{0.09\textwidth}
    \begin{subfigure}[t]{0.45\textwidth}
        \centering
        \includegraphics[width=\textwidth]{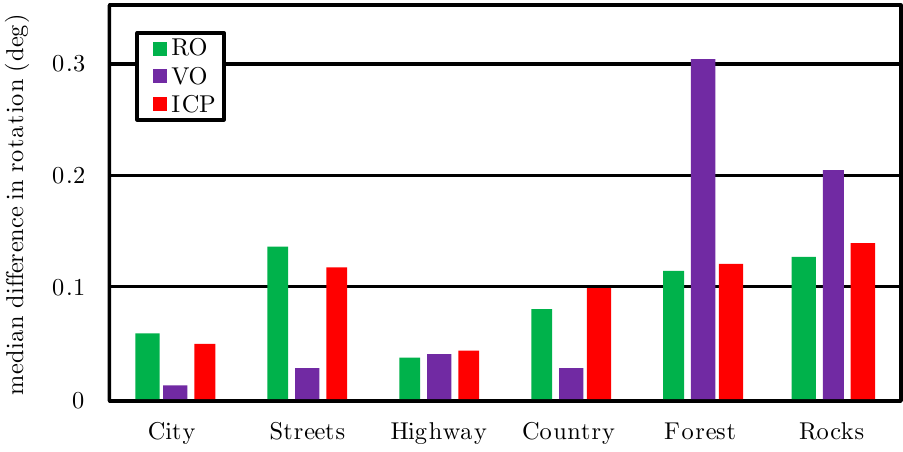}
        \caption{}
    \end{subfigure}
    \caption{Evaluation of odometry performance across six diverse settings, measured by the median difference between each method's estimated displacements relative to those of GPS in translation (a) and rotation (b). Displacement is the relative motion between consecutive scans obtained at $\sim$4 Hz. For each setting, RO (our pipeline), VO, and ICP are compared to GPS. Both RO and ICP use the keypoint sets returned by our KE algorithm. 
    }\vspace{-0.4em}\label{fig:odom_bar}
\end{figure*}

In \cite{cen2018mmwradar}, we demonstrate RO that performs comparably to VO and GPS (under ideal conditions for the latter two) and maintains precise odometry even under conditions for which VO and GPS fail. In this section, we build on this result and show that our new RO pipeline, which improves upon \cite{cen2018mmwradar}, achieves high accuracy in real time across diverse settings, from UK city streets to off-road landscapes in Iceland.

\begin{table}
\begin{center}
\caption{Datasets and their characteristics}\label{tab:diverse_environments}
\begin{threeparttable}
\begin{tabularx}{\columnwidth}{@{\extracolsep{\fill}} l  l c c c c}
\toprule
  \multirow{ 2}{*}{Dataset} & \multirow{ 2}{*}{Setting\tnote{1}} & Avg. T. & Avg. R. & Dist. & RO\\
          &  & (m)\tnote{2} & (deg)\tnote{2} & (km) & Prior\tnote{3}\\
 \midrule
 City &  Oxford, UK & 1.268 & 1.489 & 8.36 & None
 \\
 Backstreets  & Oxford, UK & 1.252 & 1.804 & 2.76 & None\\
 Highway & M40, UK & 5.164 & 1.076 & 19.6& Max.\hspace{2pt}Acc.\\
 Countryside &  Rural UK  & 2.297 & 1.264 & 16.8 & Max.\hspace{2pt}Acc.
 \\
 Forest &  Heidmörk, IS & 3.168 & 0.769 & 9.37 & None
 \\
 Rocks &  Thórsmörk, IS  & 1.267 & 1.036 & 2.35 & Max.\hspace{2pt}Acc.
 \\
\bottomrule
\end{tabularx}
\begin{tablenotes}[flushleft]\footnotesize
    \item[1] UK and IS stand for the United Kingdom and Iceland, respectively. 
    \item[2] These quantities are the average movement (in translation and rotation according to RO) between consecutive scans when the vehicle is not stationary. Scans are obtained at $\sim$4 Hz (i.e., speed is $\sim$4 times displacement).
    \item[3] ``None'' means that no motion prior is used during scan matching. ``Max. Acc.'' employs a very mild prior that limits the maximum distance the vehicle could have traveled based a maximum acceleration of 8 m/s$^2$.
    % \item[1] 4 Hz, IS = Iceland, prior, non stationary
\end{tablenotes}
\end{threeparttable}
\vspace{-0.25em}
\end{center}
\end{table}

For keypoint extraction, we use $\ell_{\text{max}} = 1000$. The six experimental datasets examined in this paper are described in Table \ref{tab:diverse_environments}. As explained in the footnotes, the final column indicates whether a motion prior is used for RO. For three datasets, no motion prior is used at all. The other three employ a very mild motion prior that uses the system's maximum possible acceleration of 8 m/s$^2$ to remove impossible matches from consideration. Empirically, our system very rarely requires this prior (e.g., 2 scan matches out of 10,000). Incorrect alignments occur when the scene structure contains high levels of symmetry, resulting in many possible alignment solutions. The lenient prior serves to reduce the number of possible solutions. Due to space constraints, odometry plots for this work can be found in our video \cite{icravid}.

\begin{table}
\begin{center}
\caption{Comparative summary of odometry methods}\label{tab:compare_odom}
\begin{threeparttable}
\begin{tabularx}{\columnwidth}{@{\extracolsep{\fill}}@{ } l c c c c}
\toprule
\vspace{-0.5pt}
 & \multicolumn{2}{c}{Translation (m)\tnote{1}}  
 & \multicolumn{2}{c}{Rotation (deg)\tnote{1}} \\
 \cmidrule(lr){2-3} \cmidrule(l){4-5}
  Comparison & Median & Std. Dev. & Median & Std. Dev. \\
 \midrule
 RO to GPS &  {0.0520} & 0.0660 & {0.0929} & 0.1632 
 \\
 RO to VO  & 0.0724 & 0.0777 & 0.1414 & 0.2105 \\
 VO to GPS\tnote{2} & 0.0577 & {0.0546} & 0.1032 & {0.1396}\\
\bottomrule
\end{tabularx}
\begin{tablenotes}[flushleft]\footnotesize
    \item[1] Statistics for difference between estimated displacement (in translation and rotation) between consecutive radar scans, which are obtained at $\sim$4 Hz. 
    \item[2] Contextualizes RO performance by     comparing two well accepted methods. 
\end{tablenotes}
\end{threeparttable}
\vspace{-0.5em}
\end{center}
\end{table}

Fig. \ref{fig:odom_bar} evaluates RO performance against GPS and compares it to state-of-the-art VO \cite{WinstonChurchill} and popular scan matching method ICP. For ICP, we use a convergence tolerance of $1$e-$5$. We improve ICP performance by providing it with a good motion prior and limiting the possible matches to nearest neighbors within 2 m. Even so, RO outperforms ICP in every setting, as shown in Fig. \ref{fig:odom_bar}(a). The occasional lower rotational error of ICP in Fig. \ref{fig:odom_bar}(b) is misleading because it only occurs when ICP's estimated translation is highly inaccurate, in which case the rotation error is meaningless. Relative to VO, RO generally achieves comparable though slightly lower accuracy, which can be explained by the fact that RO only captures 3-DOF motion compared to VO's 6-DOF. Thus, if the platform experiences rolling or pitching, RO registers this motion into its 3-DOF estimate. Similarly, if the environment is less structured, variability in elevation is discarded in radar's 2D scan, also affecting the resulting motion estimates. This reasoning is consistent with RO's slightly higher error in countryside, forest, and boulder field (i.e., rocks) datasets. RO greatly outperforms VO on the highway, on which the average travel speed is highest, making radar odometry appealing from a safety perspective. 

A summary of RO performance relative to GPS is given in Table \ref{tab:ro_summary}, and a comparison of the three odometry methods using results averaged across all datasets is given in Table \ref{tab:compare_odom}. Table \ref{tab:ro_summary} shows that our RO pipeline achieves high accuracy, differing from GPS by only 5.2 cm and 0.09 deg in translation and rotation, respectively, on average. Notably, the third row of Table \ref{tab:compare_odom} shows that the overall median difference between VO and GPS exceeds that between RO and GPS. On the other hand, the standard deviation of differences is larger for RO, meaning that RO provides highly precise yet slightly noisier motion estimates compared to VO. Importantly, as detailed in Table \ref{tab:ro_summary}, RO is remarkably accurate in common settings (e.g., city center). Although the error seemingly increases for more complex environments (e.g., in Iceland), these discrepancies may be due to factors discussed in the previous paragraph. 

This work achieves lower error in Oxford city compared to our previous work \cite{cen2018mmwradar}. Though not shown, it consistently outperforms \cite{cen2018mmwradar}, which cannot handle unstructured settings.

\begin{table}
\begin{center}
\caption{Summary of RO scan matching performance relative to GPS}\label{tab:ro_summary}
\begin{threeparttable}
\begin{tabularx}{\columnwidth}{@{\extracolsep{\fill}}@{ } l c c c c}
\toprule
 & \multicolumn{2}{c}{Translation (m)}  
 & \multicolumn{2}{c}{Rotation (deg)} \\
 \cmidrule(lr){2-3} \cmidrule(l){4-5}
  Setting & Median & Std. Dev. & Median & Std. Dev. \\
 \midrule
 City &  0.0208 & 0.0318  & 0.0597 & 0.1442 
 \\
 Backstreets & 0.0362 & 0.0398 & 0.1375 & 0.2238 \\
 Highway & 0.0480 & 0.0598 & 0.0384 & 0.0673\\
 Countryside & 0.0462 & 0.0698 & 0.0811 & 0.1380\\
 Forest & 0.0711 & 0.0899 & 0.1141 & 0.1796\\
 Rocks & 0.0897 & 0.1050 & 0.1267 & 0.2267\\
\bottomrule
\end{tabularx}
\end{threeparttable}
\vspace{-0.2em}
\end{center}
\end{table}

\section{Conclusion and Future Work} \label{sec:conc}

In this paper, we motivate the use of radar as an information-rich, long-range, and reliable sensor that is robust under challenging conditions. We propose keypoint extraction and scan matching algorithms that are designed to require minimal \textit{a priori} knowledge and human intervention. We demonstrate our pipeline's high accuracy across diverse settings, performing comparably and at times better than VO. We  discuss our approach's theoretical underpinnings, propose two performance measures, and analyze system complexity. Three areas of promising future work are: streamlining unary candidate proposal, addressing scene symmetry (e.g., using tertiary graph matching or quantifying motion observability), and deducing 3D information from 2D scans. 

% REFERENCES -----------------------
\bibliographystyle{IEEEtran}
\bibliography{IEEEabrv,ref.bib}

\end{document}